\documentclass[letterpaper, 10 pt, conference]{ieeeconf}  

\IEEEoverridecommandlockouts                              
\usepackage{amsmath,amsfonts}
\usepackage{amssymb}
\usepackage{array}
\usepackage[caption=false,font=normalsize,labelfont=sf,textfont=sf]{subfig}
\usepackage{textcomp}
\usepackage{stfloats}
\usepackage{url}
\usepackage{verbatim}
\usepackage{graphicx}
\usepackage{cite}

\usepackage[table,xcdraw]{xcolor}
\usepackage{hyperref}  
\usepackage[ruled,vlined,linesnumbered,noend]{algorithm2e}
\usepackage{mdframed}
\usepackage{fancyvrb,multirow}
\usepackage{soul}
\usepackage{dsfont,mathabx}
\usepackage{threeparttable}
\usepackage{booktabs}
\usepackage{adjustbox} 
\usepackage{diagbox}
\usepackage{booktabs}
\usepackage{multicol}
\usepackage{makecell}
\usepackage{siunitx}
\usepackage{standalone}

\newtheorem{lemma}{Lemma}

\newtheorem{definition}{Definition}
\newtheorem{remark}{Remark}
\newtheorem{example}{Example}

\title{\LARGE \bf
UMBRELLA: Uncertainty-aware Multi-robot Reactive Coordination\\
under Dynamic Temporal Logic Tasks
}

\author{Qisheng Zhao$^1$, Meng Guo$^1$, Hengxuan Du$^1$, Lars Lindemann$^2$, and Zhongkui Li$^1$
\thanks{
   The authors are with: $^1$School of Advanced Manufacturing and Robotics,
   Peking University, Beijing 100871, China; and
   $^2$Automatic Control Laboratory, ETH Zürich, Zürich 8092, Switzerland.
   This work was supported by the National Natural Science Foundation of China under grants 62425301, U2241214, T2121002.
   Corresponding author: Zhongkui Li ({\tt\small zhongkui.li@pku.edu.cn}).}
}

\begin{document}

\maketitle
\thispagestyle{empty}
\pagestyle{empty}

\begin{abstract}
Multi-robot systems can be extremely efficient for accomplishing team-wise tasks
by acting concurrently and collaboratively.
However, most existing methods either assume static task features 
or simply replan when environmental changes occur.
This paper addresses the challenging problem of coordinating multi-robot systems
for collaborative tasks involving dynamic and moving targets.
We explicitly model the uncertainty in target motion prediction
via Conformal Prediction (CP),
while respecting the spatial-temporal constraints specified by Linear Temporal Logic~(LTL).
The proposed framework (UMBRELLA) combines the Monte Carlo Tree Search~(MCTS) over
partial plans with uncertainty-aware rollouts,
and introduces a CP-based metric to guide and accelerate the search.
The objective is to minimize the Conditional Value at Risk (CVaR) of the average makespan.
For tasks released online, a receding-horizon planning scheme
dynamically adjusts the assignments based on updated
task specifications and motion predictions.
Spatial and temporal constraints among the tasks are always ensured,
and only partial synchronization is required for the collaborative tasks during online execution.
Extensive large-scale simulations and hardware experiments demonstrate substantial reductions in
both the average makespan and its variance by $23\%$ and $71\%$,
compared with static baselines.
\end{abstract}

\section{Introduction}\label{sec:intro}
Recent {advances in computation, perception and communication
enable the deployment of autonomous robots in large, remote and hazardous environments,
such as offshore drilling platforms~\cite{shukla2016application}
and construction sites~\cite{bock2007construction}.
Concurrent motion and actions can greatly improve the team efficiency~\cite{toth2002overview, cliff2015online},
while direct collaboration on a single task further extends system capabilities~\cite{varava2017herding}.}
To specify complex tasks beyond simple sequential visiting,
many studies employ formal languages such as Linear Temporal Logic (LTL)
formulas~\cite{baier2008principles}, as an intuitive yet powerful way to describe
both spatial and temporal requirements on the team behavior~\cite{kantaros2020stylus, guo2018multirobot}.
However, while tasks are commonly defined over static features such as regions and landmarks,
many real-world applications involve dynamic targets, e.g., monitoring animal flocks
or tracking moving vehicles~\cite{robin2016multi}.
These scenarios pose particular challenges for traditional
offline methods that are designed for tasks over static features~\cite{luo2022temporal, sahin2019multirobot, Liu2024fproduct},
as the future motion of targets can significantly impact the performance
and even correctness of the overall plan.

\begin{figure}[!t]
    \centering
	\includegraphics[width=3.4in]{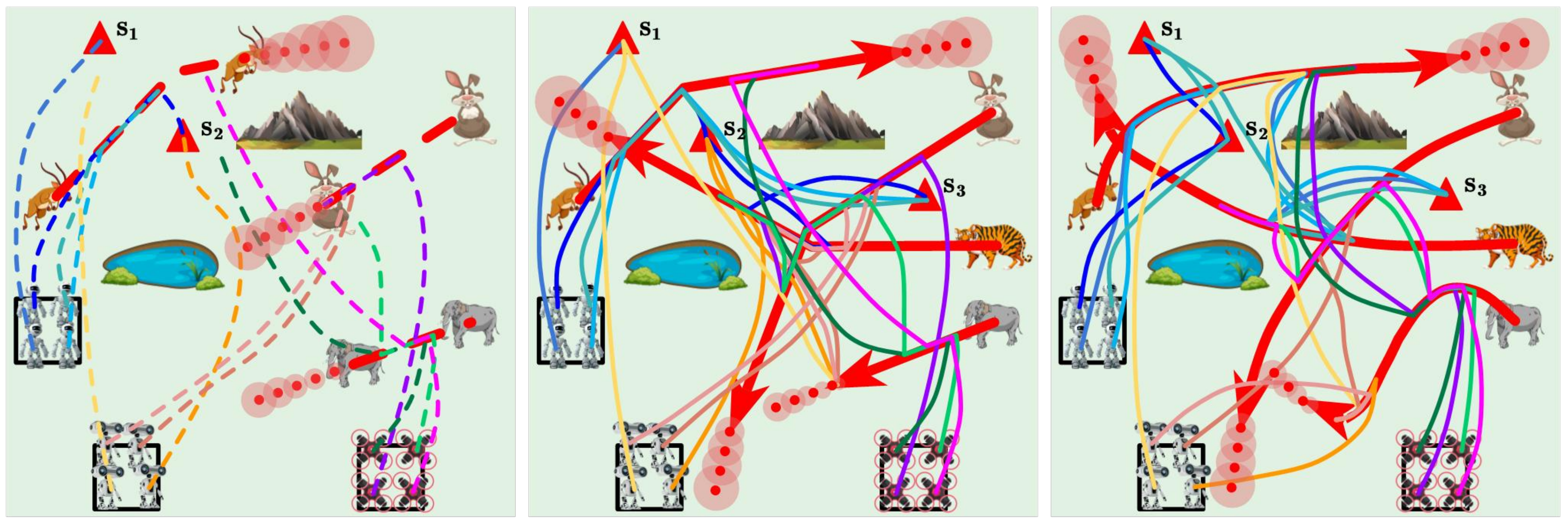}
	\hfil
	\includegraphics[width=3.4in]{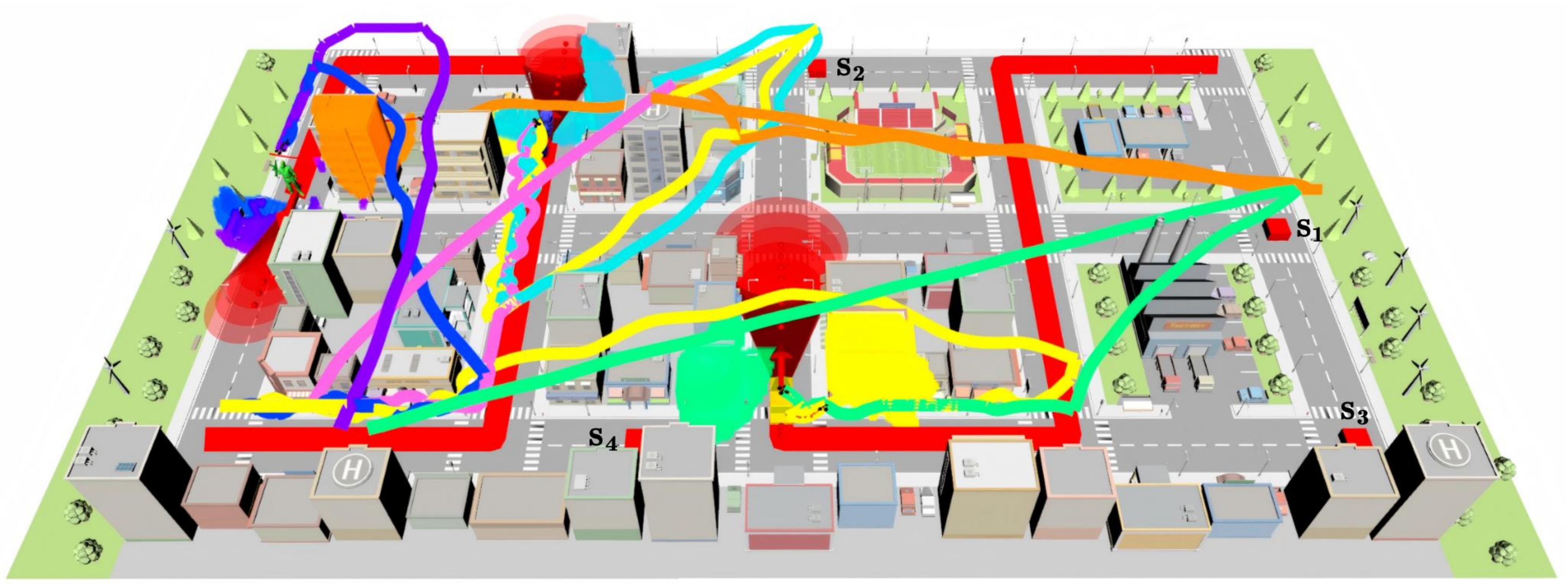}
    \hfil
    \includegraphics[width=3.4in]{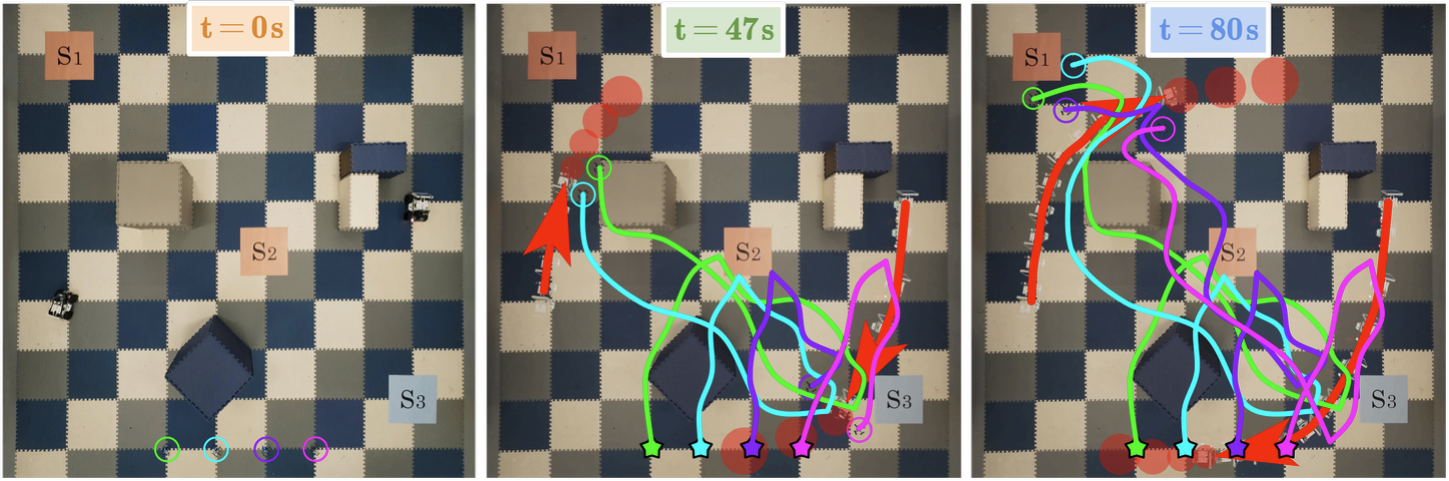}
  \vspace{-0.25in}
    	\caption{\textbf{Top:} Task plans with $12$ robots coordinating to track $4$ dynamic targets
         across $12$ tasks in two scenes (Scene-1: left and middle; Scene-2: right).
          \textbf{Middle:} ROS simulation with $8$ robots and $3$ dynamic targets executing $10$ tasks.
          \textbf{Bottom:} Hardware experiments with $4$ robots and $2$ dynamic targets performing $7$ tasks, 
          showing snapshots at different times.}
    \label{fig:online main}
    \vspace{-0.2in}
\end{figure}
\subsection{Related Work}\label{subsec:intro-related}
Extensive work has addressed task planning for team-wise temporal-logic specifications.
Centralized methods are often used to ensure optimality and completeness.
A sampling-based method in~\cite{kantaros2020stylus}
avoids synchronized products of individual models.
The work in~\cite{schillinger2018simultaneous}
decomposes team tasks into subtasks and assigns them to individual robots.
Works~\cite{luo2022temporal, sahin2019multirobot} 
formulate task constraints and assignments
as integer optimization programs.
However, these works typically assume static propositions and operate offline.

For dynamic workspaces,
\cite{guo2015multi}
performs local replanning for LTL tasks
using updated workspace models.
A reactive planning framework for heterogeneous robots
is introduced in~\cite{Zhang2024},
allowing dynamic adaptation to collaborative temporal logic missions
through local reallocation and path adjustment,
minimizing task violations.
A sampling-based coordination method in~\cite{Sama2023}
addresses environments with uncertain semantic labels and known target motions.
To handle online temporal tasks,
\cite{Liu2024fproduct} proposes computing products
of partially-ordered subtasks and reassigning them after each task update.
However, for tasks defined over dynamically moving features,
the aforementioned methods would mostly plan over the current system state,
neglecting their future motions and more importantly,
the inherent uncertainties.

Furthermore, Conformal Prediction (CP) has recently been used
in robotics to quantify uncertainties in
prediction~\cite{Lars2023CP,tonkens2023,Yu2023CP}.
In~\cite{Lars2023CP}, 
CP is integrated with model predictive control to ensure motion safety
for simple navigation tasks in dynamic environments.
Temporal correlations are predicted using CP in~\cite{tonkens2023}
to enhance long-horizon performance for single-robot navigation.
CP-based trajectory predictors for uncontrollable dynamic agents
are applied to Signal Temporal Logic (STL) control
in~\cite{Yu2023CP}.
However, the planning problem for multi-robot collaborative tasks
involving dynamic and uncertain targets remains an open challenge.

\subsection{Our Method}\label{subsec:intro-our}
This work proposes \textbf{UMBRELLA},
an online planning framework
for multi-robot coordination under dynamic temporal tasks.
It integrates CP-based motion prediction
with MCTS for task assignment.
Predicted target motions and quantified uncertainty
are used after ``Expansion'' to filter child nodes,
and incorporated into ``Simulation''
to evaluate partial plans
under the Conditional Value at Risk (CVaR) metric.
To handle online updates of target motions and newly released tasks, 
a receding-horizon scheme is adopted to dynamically adjust task assignments.
The framework ensures LTL-specified spatial-temporal constraints while accounting for target motion uncertainty,
requiring only partial synchronization during collaborative task execution.
Compared with offline baselines assuming static workspaces or simple periodic replanning,
our method significantly reduces both the mean and variance of the average makespan.

The main contribution is two-fold:
(I) the incorporation of CP-based prediction into online multi-robot coordination under complex
temporal tasks associated with dynamic targets;
(II) a substantial reduction in the average makespan for temporal tasks
specified offline and released online,
with a $23\%$ and $71\%$ decrease in mean and variance, respectively.

\section{Problem Description}\label{sec:problem}

\subsection{Robots and Targets}\label{subsec:robots}
Consider a team of~$N$ autonomous robots operating within a
workspace~$\mathcal{W}\subset \mathbb{R}^{3}$.
The state of robot~$n \in \mathcal{N} \triangleq \{1, \cdots, N\}$  at time \( t \)
is \(x_t \in \mathcal{W}\).
The maximum velocity of robot \(n\) is denoted by~$v_n$.
Each robot $n$ can execute actions from a set~$\mathcal{A}_n$.
Let~$\mathcal{A}\triangleq\bigcup_{n\in\mathcal{N}}\mathcal{A}_n$.
The team is pre-programmed with collaborations~$\mathcal{C}\triangleq \{C_1,\cdots, C_k\}$,
where each collaboration $C_k\in \mathcal{C}$ consists of a list of actions
that must be executed by different robots, denoted by:
$C_k \triangleq \left[a_1,\,a_2,\cdots,a_{\ell_k}\right]$,
where~$\ell_k>0$ is the number of required actions,
and~$a_{\ell}\in \mathcal{A}$ for $\forall \ell=1,\cdots,\ell_k$.
Each action $a_{\ell}\in C_k$ should be performed by a capable robot~$n$,
i.e., $a_\ell\in \mathcal{A}_{n}$.
Each collaboration has a fixed duration~$\rho :\mathcal{C}\rightarrow \mathbb{R}_{+}$.

Moreover, the workspace contains \(M\) dynamic targets with unknown trajectories.
The position of target~\( m \in \mathcal{M} \triangleq \{1, \cdots, M\}\) at time~\( t \) is modeled
as a random variable~\( Y_{t,m} \in \mathbb{R}^{2} \),
with the joint target state~\( Y_t \triangleq \left(Y_{t,1}, \cdots, Y_{t,M}\right) \in \mathbb{R}^{2M} \).
Their trajectories are assumed to follow an unknown distribution~\( \mathcal{D} \),
i.e., \( \left(Y_{0}, Y_{1}, \cdots \right) \sim \mathcal{D} \).
A total of \( \bar{K} \) independent trajectory samples~\( Y^{(i)} \triangleq (Y_0^{(i)}, Y_1^{(i)}, \cdots) \)
are available,
partitioned into a calibration set~\( D_{\texttt{cal}} \triangleq \{ Y^{(1)}, \cdots, Y^{(K)} \} \)
and a training set~\( D_{\texttt{tra}} \triangleq \{ Y^{(K+1)}, \cdots , Y^{(\bar{K})} \} \).
The distribution~$\mathcal{D}$ is independent of robot behaviors,
i.e., it does not depend on the team state~$x$.
  Perfect observations of target positions and velocities are continuously available,
  denoted by~$Y_{0:t} \triangleq \left(Y_{0}, \cdots, Y_{t}\right)$
  and~$V_{0:t} \triangleq \left(V_{0}, \cdots, V_{t}\right)$,
  where~$V_{t} \triangleq \left(V_{t,1}, \cdots, V_{t,M}\right)$.
  A centralized scheme aggregates all measurements
  and fuses them into real-time estimates of target states.
  For simplicity, static task regions and obstacles are modeled as targets with zero velocity.

\subsection{Task Specification}\label{subsec:task-specification}
Consider two types of atomic propositions:
(I) $p^m_n$ is \textit{true} if the distance between
robot $n \in \mathcal{N}$ and target $m \in \mathcal{M}$ is
below a given threshold.
Let $\textbf{p}\triangleq \{p^m_n,\, \forall m \in \mathcal{M},
n \in \mathcal{N}\}$;
(II) $c_k^{m}$ is \textit{true} if collaboration~$C_k$
is executed on target~\(m\).
Let $\textbf{c} \triangleq\{c_k^{m}, \forall m \in \mathcal{M},
\forall C_k \in \mathcal{C}, \forall a^c \in C_k\}$.
Thus, the complete set of propositions is denoted by~$AP \triangleq \textbf{p} \cup \textbf{c}$.

Given these propositions, 
a team-wise \emph{task} is represented as a {syntactically co-safe} LTL (sc-LTL) formula: 
$\varphi_{i} = \text{sc-LTL}(AP)$.
The overall task specification is represented as a set of LTL formulas,
which consists of two parts:
\begin{equation}\label{eq:reactive-task}
  \varphi \triangleq \varphi_\texttt{static}
  \bigwedge_{\bar{e}\in \bar{E}} \Box \left( \varphi_\texttt{obs}^{\bar{e}}
  \rightarrow \Diamond\varphi_\texttt{rep}^{\bar{e}} \right),
\end{equation}
where~\(\varphi_\texttt{static}\) is a predefined set of sc-LTL formulas;
$\bar{E}$ is the set of events triggered by online observations;
$\varphi_\texttt{obs}^{\bar{e}}$ is the propositional condition for event~$\bar{e}\in \bar{E}$;
and~$\varphi_\texttt{rep}^{\bar{e}}$ is the corresponding response task, also specified as sc-LTL.
Specifically, if $\varphi_\texttt{obs}^{\bar{e}}$ holds,
then $\varphi_\texttt{rep}^{\bar{e}}$ must eventually be satisfied.
We adopt the standard LTL syntax~\cite{baier2008principles}:
$\varphi \triangleq \top \;|\; p  \;|\; \varphi_1 \wedge \varphi_2  \;|
\; \neg \varphi  \;|\; \bigcirc \varphi  \;|\;  \varphi_1 \,\textsf{U}\, \varphi_2,$
where $\top\triangleq \texttt{True}$, $p \in AP$, $\bigcirc$ (\textit{next}),
$\textsf{U}$ (\textit{until}) and $\bot\triangleq \neg \top$.
sc-LTL~\cite{belta2017formal} is a fragment of LTL restricted to operators~\( \bigcirc \), $\textsf{U}$,
and $\Diamond$ (eventually), written in positive normal form without the
negation operator $\neg$ preceding temporal operators.

Moreover, the task plan for the robot team is defined as
\begin{equation}\label{plan}
    \Pi_{\mathcal{N}} \triangleq (\pi_1, \pi_2, \cdots, \pi_N),
\end{equation}
where \(\pi_n \triangleq (t^1_n,a^1_n,m^1_n)(t^2_n, a^2_n,m^2_n)
\cdots (t^{K_n}_n, a^{K_n}_n,m^{K_n}_n) \)
is the timed sequence of actions and corresponding targets for robot \(n\in \mathcal{N}\).
Each action \(a^k_n \in \mathcal{A}_n\) is executed on target \(m^k_n\in \mathcal{M}\)
at time~$t^k_n>0$,
where $k \in \mathcal{K}_n\triangleq \{1,\cdots, K_n\}$.
Given~$\Pi_{\mathcal{N}}$, the induced trace is given by the sequence of propositions
satisfied by the actions, i.e.,
\(w_\Pi = \sigma_1 \sigma_2 \cdots \sigma_L\)
where~\( \sigma_\ell \in 2^{\textit{AP}} \).
The language of~\( \varphi \) is
$\mathcal{L}_\varphi \triangleq \{\boldsymbol{w}\,|\,w\models\varphi\}$,
where $\models$ is the satisfaction relation.
Since sc-LTL formulas admit satisfaction by finite traces~\cite{belta2017formal,baier2008principles},
the plan satisfies \(\varphi\) if $w_\Pi\in \mathcal{L}_\varphi$, denoted by \(\Pi_\mathcal{N} \models \varphi\).
The \emph{makespan}~$T_\varphi$ is defined as the minimal time to generate a trace satisfying~$\varphi$.
As the tasks assigned to the robot system consist of initially issued tasks
and online triggered tasks during execution,
for a sequence of tasks $\varphi(t) \triangleq \{ \varphi_1, \ldots, \varphi_L \}$ received up to time~$t$, 
the \emph{average makespan} is defined as:
$\overline{T}_\varphi \triangleq \frac{1}{L} \sum_{\ell = 1}^{L} T_{\varphi_\ell}.$


\subsection{Problem Statement}\label{subsec:prob-statement}

Given task formulas~$\varphi(t)$ and observations~$\{Y_{0:t}, V_{0:t}\}$,
the overall objective is to synthesize the team-wise plan~\(\Pi_\mathcal{N}\)
to satisfy~\(\varphi(t)\) and minimize
the average makespan~\(\overline{T}_{\varphi}\).
\begin{example}\label{exp:task}
Consider a fleet of UAVs and ground robots (GRs) deployed for wildlife monitoring
in a workspace labeled with \texttt{poacher}, \texttt{trap} and \texttt{animal}.
An offline collaborative task~$\varphi_\texttt{1}$ is specified as:
\begin{equation}
\label{eq:task}
\begin{aligned}
    \varphi_\texttt{1} = \varphi_{\texttt{p}-\texttt{s}_1} \wedge \varphi_{\texttt{mf}-\texttt{a}}, \:
    \varphi_{\texttt{p}-\texttt{s}_1} = \Diamond \texttt{patrol}_{\texttt{s}_1},\\
    \varphi_{\texttt{mf}-\texttt{a}} = \Diamond(\texttt{monitor}_\texttt{a}\wedge\neg\texttt{film}_\texttt{a}\wedge\Diamond\texttt{film}_\texttt{a}), \\
\end{aligned}
\end{equation}
which requires to patrol region \(\texttt{s}_1\),
monitor and film antelopes in sequence.
GRs must avoid obstacles in the workspace.
The reactive protocol specifies:
$(\texttt{poacher}\rightarrow\Diamond \texttt{arrest})$,
where an ``arrest'' task is triggered upon detecting a poacher,
and $(\texttt{trap}\wedge\texttt{animal}\rightarrow\Diamond \texttt{rescue})$,
triggers a ``rescue'' task when an animal is trapped.
\hfill $\blacksquare$
\end{example}

\section{Proposed Solution}\label{sec:solution}
\begin{figure*}[!t]
  \centering
  \includegraphics[width=7.0in]{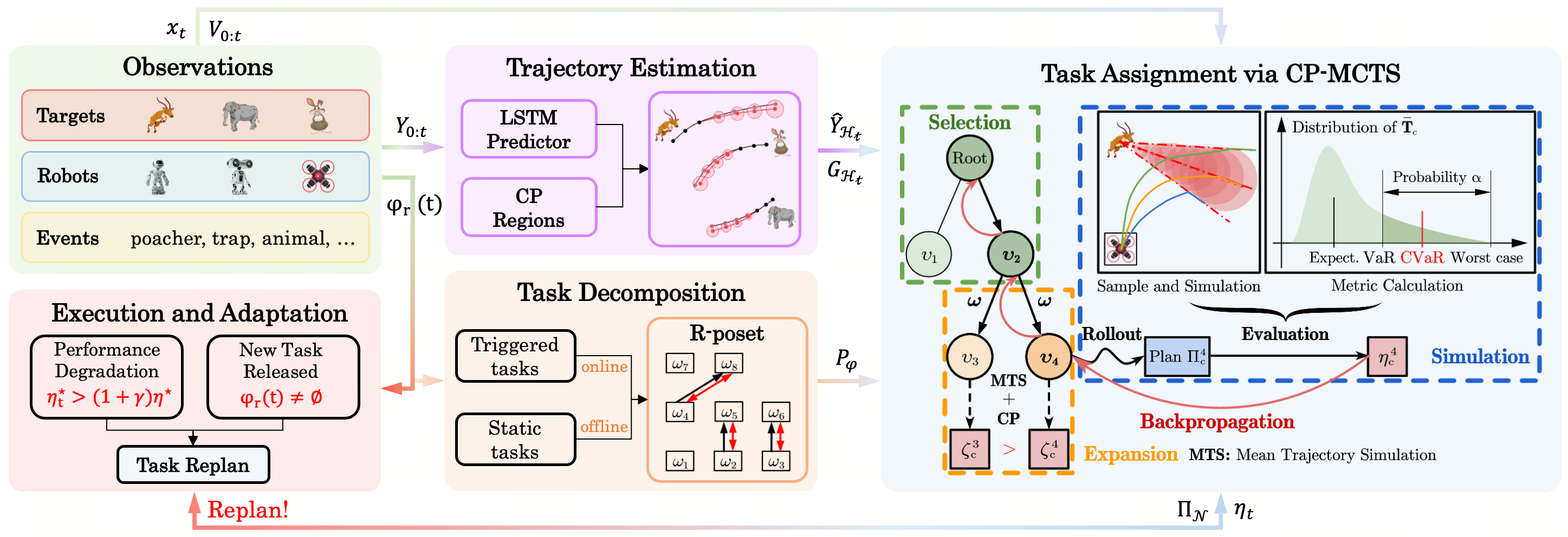}
  \vspace{-6mm}
  \caption{Overview of the proposed framework, 
  consisting of four main components: 
  (i) trajectory estimation via LSTM and CP, 
  (ii) task decomposition into an R-poset, 
  (iii) CP-MCTS for uncertainty-aware assignment, 
  and (iv) online execution and receding-horizon adaptation. 
  In the R-poset illustration, 
  precedence and mutual-exclusion relations are marked by black and red arrows, 
  respectively.}
  \label{fig:fra}
  \vspace{-0.2in}
\end{figure*}
As illustrated in Fig. \ref{fig:fra}, the proposed solution consists of four components:
the dynamic target trajectory estimation in Sec. \ref{subsec: trajectory};
the derivation of the relaxed partially-ordered set (R-poset)
in Sec. \ref{subsec: poset};
the uncertainty-aware task assignment via CP-MCTS in Sec. \ref{subsec: assignment};
and 
a receding-horizon planning scheme for real-time adaptation
in Sec. \ref{subsec: execution}.

\subsection{Trajectory Estimation of Dynamic Targets}\label{subsec: trajectory}

\subsubsection{Trajectory Predictor}
Given a task horizon~$T_{{\varphi}}$ and the history of target observations $Y_{0:t}$,
we train an independent trajectory predictor for each target
based on the training dataset~$D_{\texttt{tra}}$.
Let $D_{\texttt{tra}}^m \subset D_{\texttt{tra}}$ collect trajectories of target~$m$,
$Y_{0:T,m}^{(i)} \triangleq (Y_{0,m}^{(i)}, \cdots, Y_{t,m}^{(i)}, Y_{t+1,m}^{(i)}, \cdots, Y_{T_{{\varphi}},m}^{(i)})$
denotes the $i$-th trajectory in $D_{\texttt{tra}}^m$.
The trajectory predictor is defined as
$\Upsilon_m : \mathbb{R}^{2(t+1)} \rightarrow \mathbb{R}^{2(T_{{\varphi}}-t)}$ that estimates the future states of target $m$
as $\widehat{Y}_{\mathcal{H}_t,m} \triangleq \Upsilon_m (Y_{0:t,m}),
\mathcal{H}_t \triangleq \{t+1, \cdots, T_{{\varphi}}\}$, where
$\Upsilon_m (Y_{0:t,m}) \triangleq (\widehat{Y}_{t+1 \mid t,m}, \cdots, \widehat{Y}_{T_{{\varphi}} \mid t,m})$.
Let $\mathbf{\Upsilon} \triangleq (\Upsilon_1, \cdots, \Upsilon_M)$.
In principle, any trajectory predictor can be employed,
including long short-term memory (LSTM) networks~\cite{Hochreiter1997LSTM}, recurrent neural networks (RNN)~\cite{rnn}, and gated recurrent units (GRU)~\cite{gru}.
In this work, a LSTM network is adopted for each target and trained
by minimizing the following loss function:
$$\underset{\Upsilon_m}{\textbf{min}} \frac{1}{|D_{\texttt{tra}}^m|} \sum_{i=1}^{|D_{\texttt{tra}}^m|} \left\| Y_{\mathcal{H}_t,m}^{(i)} - \Upsilon_m(Y_{0:t,m}^{(i)}) \right\|^2,$$
which is the Mean Squared Error (MSE) over the predicted trajectory for target $m$.
Stacking per-target predictions gives
$\widehat{Y}_{\mathcal{H}_t} \triangleq (\widehat{Y}_{\mathcal{H}_t,1}, \cdots, \widehat{Y}_{\mathcal{H}_t,M})$.
Furthermore, based on the velocity measurements~$V_{0:t}$,
we define \(v^{\ast}_{m}(t) \triangleq \max \left\{ |V_{\varsigma , m}| : \varsigma \leq t \right\}\)
as the maximum historical velocity of target \(m\) up to time \(t\).
This velocity bound serves as an input parameter for the subsequent planning module.

\subsubsection{Conformal Prediction Regions}\label{subsec:cpr}
The CP framework can construct regions around predicted trajectories
that contain the true trajectory with high probability,
see~\cite{Angelopoulos2021CP} for detailed descriptions.
More specifically, we adopt the method in~\cite{Lars2023CP} to construct valid prediction regions
for each target independently.
Given observations $Y_{0:t}$ at time $t$,
the trajectory predictors~$\mathbf{\Upsilon}$ generate predictions~$\widehat{Y}_{\mathcal{H}_t}$
for the task horizon $T_{{\varphi}}$.
For each target $m \in \mathcal{M}$, given a failure probability~$\delta \in (0,1)$,
the prediction regions
$G_{\mathcal{H}_t,m} \triangleq (G_{t+1 \mid t,m}, \cdots, G_{T_{\varphi} \mid t,m})$
are constructed such that:
\begin{equation}
\label{eq:cpg}
\text{Pr}\Big(\left\| Y_{h, m} - \widehat{Y}_{h \mid t, m} \right\| \leq G_{h \mid t, m}, \Big.\Big. 
\forall h \in \mathcal{H}_t \Big) \geq 1 - \delta,
\end{equation}
where \(G_{h \mid t, m}\) denotes the \(h\)-step prediction error for target~$m$ at time~$t$.
The \textit{nonconformity score} is defined as
\begin{equation}
\label{eq:ncs}
R_{h|t, m}^{(i)} \triangleq \left\| Y_{h, m}^{(i)} - \widehat{Y}_{h|t, m}^{(i)} \right\|,
\end{equation}
for calibration trajectories \( Y^{(i)} \in D_{\texttt{cal}}^m \),
where $D_{\texttt{cal}}^m$ is the calibration dataset for target $m$.
Specifically, predictions~$\widehat{Y}_{h|t, m}^{(i)}$ are computed
for each $Y^{(i)} \in D_{\texttt{cal}}^m$,
and the corresponding nonconformity scores $R_{h|t, m}^{(i)}$ are calculated.
These scores are sorted in non-decreasing order after adding $R_{h|t, m}^{(|D_{\texttt{cal}}^m|+1)} \triangleq \infty$.
The error bound~$G_{h|t, m}$ is then chosen as the $p$-th smallest score, where $p = \lceil (|D_{\texttt{cal}}^m|+1)(1-\bar{\delta}) \rceil$.
Finally, the prediction regions for all targets are aggregated as:
$G_{\mathcal{H}_t} \triangleq (G_{h \mid t,1}, \cdots, G_{h \mid t,M})$.



\subsection{R-posets for Task Formulas}\label{subsec: poset}
 To efficiently capture the temporal constraints embedded 
 in a sc-LTL formula~$\varphi$, 
 we adopt the notion of relaxed partially ordered sets (R-posets) 
 as proposed in~\cite{LIU2024111377}. 

 Given $\varphi$,
 we first translate it into a Nondeterministic Büchi Automaton (NBA)
 $\mathcal{A}_\varphi = (S,\,\Sigma,\,\delta,\,S_0,\,S_F)$,
 where~$S$ is the set of states;
 $\Sigma=AP$ is the alphabet;
 $\delta:S\times \Sigma\rightarrow2^{S}$ is the transition relation;
 $S_0, S_F\subseteq S$ denote the initial and accepting states.
 Along a satisfying run of $\mathcal{A}_\varphi$, 
 a \textit{subtask}~$\omega$ is defined as
 the minimal symbol enabling a transition between states.
 Each subtask represents a unit of progress,
 and the collection of all subtasks forms~$\Omega_\varphi$.

\begin{definition}[R-poset]
	An R-poset over $\varphi$ is defined as the triple:
	$P_\varphi=(\Omega_\varphi,\preceq_\varphi,\neq_\varphi)$:
	(I) $\Omega_\varphi$ is the set of subtasks;
    (II) $\preceq_{\varphi}\subseteq \Omega_{\varphi} \times \Omega_{\varphi}$ 
	is the precedence relation:
	if $(\omega_1, \omega_2) \in \preceq_\varphi$, then $\omega_2$ cannot start before $\omega_1$ starts;
	(III) $\neq_\varphi \subseteq 2^{\Omega_\varphi}$ is the mutual exclusion relation: 
	subtasks in the same set cannot be executed simultaneously.
\hfill $\blacksquare$
\end{definition}

Although an R-poset is not unique for a given $\varphi$, the set of all possible R-posets \(P_{\varphi}\) is as expressive as the original NBA.
Any plan consistent with an R-poset must satisfy $\varphi$, since $\text{Words}(P_{\varphi}) \subset \text{Words}(\varphi)$.
In practice, we construct $P_{\varphi}$ using the algorithm in~\cite{LIU2024111377},
denoted as \(\texttt{Compute\_poset}(\cdot)\).
An example of an R-poset is illustrated in Fig.~\ref{fig:fra}.


\subsection{Uncertainty-Aware Task Assignment}\label{subsec: assignment}
Given the estimation of target motions~\((\widehat{Y}_{\mathcal{H}_t}, G_{\mathcal{H}_t})\)
and the final R-poset~$P_{\varphi}=(\Omega_{\varphi},\preceq_{\varphi},\neq_{\varphi})$,
the objective is to find an efficient assignment of
all subtasks in~$\Omega_{\varphi}$ given the robot team~$\mathcal{N}$ such that
all partial orders in~$\preceq_{\varphi},\, \neq_{\varphi}$ are respected
and the average makespan of all tasks \(\overline{T}_{\varphi}\) is minimized.


\subsubsection{CP-based Monte Carlo Tree Search}\label{sec:cp-mcts}

Monte Carlo Tree Search (MCTS) is a well-known heuristic search algorithm
for solving complex planning problems in dynamic scenes.
Built upon this algorithm, this work introduces CP-based Monte Carlo Tree Search (CP-MCTS).
As summarized in Alg.~\ref{alg:cp-mcts} and Fig.~\ref{fig:fra},
it is a centralized task planning algorithm
designed to efficiently
handle complex temporal tasks associated with dynamic targets.
It repeats four stages until the time budget expires:
selection, expansion, simulation, and backpropagation.
Notably, as the number of robots and tasks increases,
the nodes generated during the ``Expansion'' phase grow rapidly.
If ``{Simulation}'' is performed for all expanded nodes,
the algorithm becomes biased toward breadth-first search, neglecting depth exploration and thus reducing efficiency.
To mitigate this, a CP-based metric is introduced to
efficiently evaluate and select the most promising child nodes,
which are then advanced to the ``Simulation'' stage.

\subsubsection{Selection and Expansion}
Each node in the search tree represents a partial assignment of subtasks, i.e.,
$\nu \triangleq (\tau_1,\,\tau_2,\cdots,\tau_N)$,
where $\tau_n$ is the ordered sequence of subtasks assigned
to robot~$n\in \mathcal{N}$.

\setlength{\textfloatsep}{5pt}
\begin{algorithm}[!t]
    \caption{CP-based MCTS \texttt{(CP-MCTS)}}
    \label{alg:cp-mcts}
    \SetKwInOut{Input}{Input}
    \SetKwInOut{Output}{Output}
    \Input{Robots~$\mathcal{N}$, poset~$P_{\varphi}$, duration func.~$\rho$, \\
	time budget~$t_b$, target estimations~$\widehat{Y}_{\mathcal{H}_t}, {G}_{\mathcal{H}_t}$.}
    \Output{Plan~$\Pi_\texttt{c}^\star$, average makespan~$\eta_\texttt{c}^\star$.}
    Initialize root node $\nu_0$, \(\eta_\texttt{c}^\star \leftarrow \infty\);\\
    \While{$time < t_b$}{
    	Leaf node~$\nu \leftarrow$ \texttt{Selection}($\nu_0$);\\
	Child nodes $\{ \nu_{+}\}$ $\leftarrow$ \texttt{Expansion}($\nu,\, P_{\varphi}$);\\
    Filtered nodes $\{ \nu_\texttt{s}\}$ via $\zeta$ in (\ref{eq:metric});\\
	\For{$\nu_\texttt{s} \in \{\nu_\texttt{s}\}$}
	{\tcc{Simulation\,=\,Rollout\,+\,Eval}
    $\Pi_\texttt{c} \leftarrow$ \texttt{Rollout}($\nu_\texttt{s},\, P_{\varphi},\, \rho,\, \widehat{Y}_{\mathcal{H}_t},\, x_t$);\\
	$\eta_\texttt{c} \leftarrow \texttt{Eval}(\Pi_\texttt{c},\, P_{\varphi},\, \rho,\, \hat{Y}_{\mathcal{H}_t},\, G_{\mathcal{H}_t},\, x_t)$;\\
	\If{$\eta_\texttt{c} < \eta_\texttt{c}^\star$}
	{$\eta_\texttt{c}^\star \leftarrow \eta_\texttt{c},\, \Pi_\texttt{c}^\star \leftarrow \Pi_\texttt{c}$;\\}
	\texttt{Backpropagate}($\nu_\texttt{s}, \, \xi_\texttt{c}$);\\
	}
    }
\end{algorithm}

During~``{Selection}'',
the Upper Confidence Bound applied to Trees (UCT) \cite{2006Bandit} is used to choose
nodes for ``{Expansion}''.
The UCT value is computed as
$(\overline{\xi}_i+Q\times\sqrt{\frac{ln B}{b_i}})$,
where \(\overline{\xi}_i\) is the estimated value of the \(i\)-th child node,
\(b_i\) is its visit count,
\(B\) is the visit count of the current node,
and \(Q>0\) is a parameter that balances exploration and exploitation.
Starting from the root, at each level the child with the highest UCT value 
is recursively selected until a leaf node is reached.

During~``{Expansion}'', unless a terminal state with a complete assignment is reached,
child nodes are generated from the leaf node by assigning the \textit{next} subtask
to the robot team.
Let $\Omega_{\nu}\triangleq \{\omega\in\tau_n,\,\forall n\in \mathcal{N}\}$ be
the set of subtasks already assigned in node~$\nu$,
and $\Omega^-_{\nu} \triangleq \Omega_{\varphi} \backslash \Omega_{\nu}$
the remaining ones.
To ensure feasibility,
the \textit{next} subtask~$\omega$ is chosen from 
$\Omega^{\texttt{a}}_{\nu} \triangleq \big{\{}\omega\, |\, \omega \in \Omega^-_{\nu}, \omega' \in \Omega_\nu, \forall \omega' \in \text{Pre}(\omega)\big{\}}$,
where 
$\text{Pre}(\omega)$ represents the set of subtasks that must be completed
before~$\omega$ according to R-poset~$P_{\varphi}$.
In other words, a subtask~$\omega_i$ cannot be assigned to node~$\nu$
if some $\omega_j$ with $(\omega_j, \omega_i) \in \preceq_{\varphi}$
has not yet been assigned.
When a subtask~$\omega \in \Omega^{+}_{\nu}$ is selected,
a child node $\nu^+$ is created by assigning $\omega$ to a robot group~$\mathcal{I}_{\omega}$,
i.e., by appending it to the local plan~$\tau_n$ of each robot~$n \in \mathcal{I}_{\omega}$.
For child node~$\nu^+$, define \emph{key subtasks} $\Omega^{\texttt{s}}_{\nu^+}\subset \Omega_{\nu^+}$ as:
$\Omega^{\texttt{s}}_{\nu^+}\triangleq \big{\{} \omega\,|\,\omega=\tau_{n}[0],
\, \omega \notin \tau_{n'}[1:],\, \forall n,n' \in \mathcal{N}\big{\}}$,
where~$\Omega^{\texttt{s}}_{\nu^+}$ contains the first-to-execute subtasks across robots,
independent of others.
Stepwise simulation is then performed based on
the predicted target trajectories~\(\widehat{Y}_{\mathcal{H}_t}\),
as described in the sequel,
resulting in the predicted completion time~$\widehat{T}_{\omega}$
for each assigned subtask~$\omega \in \Omega_{\nu^+}$.


\subsubsection{Simulation}
The ``{Simulation}'' procedure consists of three steps:
(I) \textit{rollout} to complete an assignment;
(II) sample the makespan distribution;
(III) evaluate the plan based on this distribution.
First, a \textit{rollout} policy is applied recursively from the selected child node
until all subtasks are assigned.
In each iteration, the \textit{next} subtask is selected as in ``Expansion''
and assigned to a robot group either randomly 
or greedily according to stepwise simulation results.
To enhance \textit{rollout} diversity,
we use a random factor \(\epsilon \in [0, 1]\): 
with probability \(\epsilon\), 
a feasible robot group is chosen uniformly at random;
with probability \(1-\epsilon\),
the group expected to initiate this subtask earliest is selected.
Once a complete plan~\(\Pi_\texttt{c}\) is obtained, 
a sampling-based method derives its average makespan distribution~\(\mathbf{\overline{T}}_\texttt{c}\)
via stepwise simulation with \(z\in \mathbb{N}\) samples drawn
from the prediction regions \(\widehat{Y}_{\mathcal{H}_t}, G_{\mathcal{H}_t}\).
\begin{definition}[VaR and CVaR]
The value at risk (VaR) at risk level \(\alpha \in (0,1]\)
is defined as $\text{VaR}_{\alpha}(\Pi_\texttt{c}) \triangleq
      {\textbf{inf}}\big\{\rho\in\mathbb{R},\, \text{Prob}(
      \mathbf{\overline{T}}_\texttt{c}\geq\rho)\geq\alpha\big\}$,
i.e., the \(\alpha\)-quantile of the distribution \(\mathbf{\overline{T}}_\texttt{c}\).
The associated conditional value at risk (CVaR) is defined as:
$\text{CVaR}_\alpha(\Pi_\texttt{c})
\triangleq \mathbb{E}\big{[}\mathbf{\overline{T}}_\texttt{c}|
  \mathbf{\overline{T}}_\texttt{c}\geq\mathrm{VaR}_\alpha(\Pi_c)\big{]}$,
as the expected value of the worst \(\alpha\)-quantile of \(\mathbf{\overline{T}}_\texttt{c}\).
\hfill $\blacksquare$
\end{definition}

As shown in Fig.~\ref{fig:fra},
VaR captures the maximum loss at a given confidence level,
while CVaR assesses the expected loss beyond the VaR threshold,
providing a more informative measure of tail risk.
Accordingly, we use \(\eta_\texttt{c}\triangleq \text{CVaR}_{\alpha}(\Pi_\texttt{c})\)
as the evaluation metric for each plan based on its average makespan distribution.
The procedure of deriving the distribution and computing the CVaR
is encapsulated as~\(\texttt{Eval}(\cdot)\).
During the search, the best plan~$\Pi_\texttt{c}^\star$
and its minimal risk value~$\eta_\texttt{c}^\star$ are maintained.
For each new candidate~$\Pi_\texttt{c}$,
if \(\eta_\texttt{c} < \eta_\texttt{c}^\star\),
both $\Pi_\texttt{c}^\star$ and $\eta_\texttt{c}^\star$ are updated.

\subsubsection{Backpropagation}
To support the tree search, we define a normalized performance measure:
\begin{equation}
\xi_\texttt{c} \triangleq 2 - \frac{\eta_\texttt{c}}{\eta_\texttt{c}^\star},
\end{equation}
which facilitates comparison across different branches of the search tree.
During ``{Backpropagation}'',
the evaluation values and visit counts of nodes are propagated and
updated along the path from the selected node to the root.

\begin{lemma}
\textnormal{
Given an expanded node~$\nu^+$,
the completion times of \textit{key subtasks} in $\Omega^{\texttt{s}}_{\nu^+}$
satisfy that
$$\text{Pr}\Big(T_{\omega} \leq \widehat{T}_{\omega} + \frac{G_{\widehat{T}_{\omega}, m}}
       {\underset{n \in \mathcal{I}_{\omega}}{\textbf{min}}\{v_n\}-v^{\star}_{m}}, \Big.
\Big. \forall \omega \in \Omega^{\texttt{s}}_{\nu^+} \Big) \geq (1-\delta)^{|\Omega^{\texttt{s}}_{\nu^+}|},$$
where $T_{\omega}$ is the actual completion time of subtask~$\omega\in \Omega^{\texttt{s}}_{\nu^+}$;
$m$ is the associated target;
$\mathcal{I}_{\omega}$ is the set of robots executing~$\omega$;
$v_{n}$ and $v^{\star}_{m}$ are the velocities of robot $n$ and target $m$.}
\end{lemma}
\begin{proof}
From the simulation results,
all robots $n \in \mathcal{I}_{\omega}$ assigned to subtask $\omega \in \Omega^{\texttt{s}}_{\nu^+}$
would reach the predicted position
$\widehat{Y}_{\widehat{T}_{\omega}, m}$ of target~$m$
by time instance~$\widehat{T}_{\omega}$.
By the probability guarantee in~\eqref{eq:cpg}, it holds that
$$\text{Pr}\Big(\left\| Y_{\widehat{T}_{\omega}, m} - \widehat{Y}_{\widehat{T}_{\omega}, m} \right\|
\leq G_{\widehat{T}_{\omega}, m}, \Big.
\Big. \forall \omega \in \Omega^{\texttt{s}}_{\nu^+} \Big)
\geq (1 - \delta)^{|\Omega^{\texttt{s}}_{\nu^+}|}.$$
The additional delay required to reach target~$m$ is upper-bounded by
$\widehat{T}_\omega^{\texttt{e}} = {G_{\widehat{T}_\omega, m}}
/({{\textbf{min}_{n \in \mathcal{I}_{\omega}}}\{v_n\} -v^{\star}_{m}})$,
this completes the proof.
\end{proof}


The actual completion time of a first executed subtask~\(T_{\omega}\)
is probabilistically bounded by the predicted time \(\widehat{T}_{\omega}\)
plus an uncertainty term determined by the prediction region~\(G_{\widehat{T}_{\omega}, m}\)
and the velocity difference between the slowest robot and the target.
Hence, with confidence level \((1-\delta)\),
each \textit{key subtask} completion time can be reliably estimated.
The evaluation metric for a child node~$\nu^+$ is then defined as:
\begin{equation}
\label{eq:metric}
\zeta_{\nu^+} \triangleq \frac{\sum_{i=1}^{|\Omega_{\nu^+}|}\widehat{T}_{\omega_i}  + \sum_{j=1}^{|\Omega^{\texttt{s}}_{\nu^+}|} \widehat{T}_{\omega_j}^\texttt{e} }{|\Omega_{\nu^+}|},
\end{equation}
which approximates the average completion time of assigned subtasks with uncertainty adjustment.
After ``Expansion'', for each \textit{next} subtask, the child node with the smallest~$\zeta$ is selected to enter ``{Simulation}'',
forming the filtered set~$\{\nu_\texttt{s}\}$.

\begin{remark}\label{rm:velocity}
For a child node $\nu^+$, if a subtask $\omega \in \Omega^{\texttt{s}}_{\nu^+}$
is assigned to a
robot~$n \in \mathcal{I}_{\omega}$ with velocity~$v_n \leq  v^{\star}_{m}$,
the robot may be unable to complete the subtask.
In this case, the predicted completion time is $\widehat{T}_{\omega} = \infty$,
which yields $\zeta_{\nu^+} = \infty$;
thus the node is excluded from the ``Simulation'' step.
\hfill $\blacksquare$
\end{remark}
\subsection{Online Execution and Adaptation}\label{subsec: execution}
\subsubsection{Online Execution and Adaptation}
\setlength{\textfloatsep}{5pt}
\begin{algorithm}[!t]
    \caption{Online Dynamic Task Assignment}
    \label{alg:ta}
    \SetKwInOut{Input}{Input}
    \SetKwInOut{Output}{Output}
    \Input{Robots \(\mathcal{N}\),\, task formula~$\varphi(t)$,\, duration func.~$\rho$,\,
    	      trajectory predictors $\mathbf{\Upsilon}$,\, prediction regions $G_{\mathcal{H}_t}$,\,
              observations $\{Y_{0:t}, V_{0:t}\}$.}
    \Output{Assignment~$\Pi^\star$.}
    Initialize \(P_{\varphi}\) \(\leftarrow \texttt{Compute\_poset}(\varphi(0))\),\,
    \( \Omega_\iota \leftarrow \emptyset\); \\
    Initialize \(\Pi^\star, \, \eta^\star \leftarrow \texttt{CP-MCTS}(P_{\varphi}, \rho,
    \hat{Y}_{\mathcal{H}_0}, G_{\mathcal{H}_0}, x_0)\);\\
    \While{not terminated}
        {
        Each robot \(n \in \mathcal{N}\) applies \(\pi_n\);\\
        Sense $x_{t}$ and $Y_{t}$;\\
        Obtain predictions $\hat{Y}_{\mathcal{H}_t}$;\\
        Update \(\varphi(t)\) by (\ref{eq:task-update});\\
        Update \(P_{\varphi}\) by \(\texttt{Compute\_poset}(\varphi(t))\);\\
        \If{\(\varphi_{\texttt{r}}(t)\)}
        {
             \(\Pi^\star,\, \eta^\star\) \(\leftarrow\) \texttt{CP-MCTS}(\(P_{\varphi}, \rho, \hat{Y}_{\mathcal{H}_t}, G_{\mathcal{H}_t}, x_t\));\\
        }
        \Else{
       	 \(\eta^\star \leftarrow \eta^\star - 1\);\\
            \If{\(\Omega_{\texttt{c}}(t) \neq \emptyset\)}
            {
                Update \(\eta^\star\) by (\ref{eq:update_eta});\\
                \(\Omega_\iota \leftarrow \Omega_\iota \cup \Omega_\texttt{c}(t)\);\\
            }
            \(\eta^\star_t\leftarrow \texttt{Eval}(\Pi^\star, P_{\varphi}, \rho, \hat{Y}_{\mathcal{H}_t}, G_{\mathcal{H}_t}, x_t)\);\\
            \If{\(\eta^\star_t > (1+\gamma)\eta^\star\)}
            {
                \(\widehat{\Pi}_t, \, \widehat{\eta}_t \leftarrow \texttt{CP-MCTS}(P_{\varphi}, \rho, \hat{Y}_{\mathcal{H}_t}, G_{\mathcal{H}_t}, x_t)\);\\
                \If{\(\widehat{\eta}_t < \eta^\star_t\)}
                {
                    \(\Pi^\star \leftarrow \widehat{\Pi}_t,\, \eta^\star \leftarrow \widehat{\eta}_t\);\\
                }
            }
            }
        \(t \leftarrow t+1;\)
    }
\end{algorithm}

As illustrated in Fig.~\ref{fig:fra} and summarized in Alg.~\ref{alg:ta},
the planning scheme consists of two stages:
initial planning and online adaptation.
Following the initial plan,
each robot \(n \in \mathcal{N}\) executes its local plan~\(\pi_n\)
by navigating to the targets specified by the assigned subtasks
and then performing the corresponding actions.
When executing action~\(a^{K_n}_n\) on target~\(m^{K_n}_n\),
the robot~$n$ must wait for collaborators
if \(a^{K_n}_n\) belongs to a multi-robot collaboration~$C_k$ (i.e., \(|C_k|>1\)).
Such partial synchronization is essential to handle uncertainties in navigation and task execution times.
If the action is non-collaborative,
the robot proceeds independently without waiting.

At each time step \(t>0\),
the robot states $x_t$ and target observations~$Y_t$ are updated,
and new predictions~$\hat{Y}_{\mathcal{H}_t}$ are generated from $Y_{0:t}$
using the trajectory predictors~$\mathbf{\Upsilon}$.
Moreover, the task formula \(\varphi(t)\) is updated
by integrating triggered reactive task formula, i.e.,
\begin{equation}
\label{eq:task-update}
    \varphi(t) \triangleq \varphi(t^-)
  \wedge \varphi_{\texttt{r}}(t),
\end{equation}
where \(\varphi(t^-)\) is the previous formula
and \(\varphi_{\texttt{r}}(t)\triangleq \bigwedge_{\bar{e}\in \bar{E}}
 \Diamond\varphi_\texttt{rep}^{\bar{e}}(t)\) is the reactive tasks released at time~$t$.
The current R-poset \(P_{\varphi}\) is updated accordingly.
Replanning is triggered under two conditions:
(I) new reactive tasks appear, i.e., $\varphi_{\texttt{r}}(t) \neq \emptyset$, or
(II) the current plan's performance degrades significantly, i.e.,
$\eta^\star_t > (1+\gamma)\eta^\star$, where $\gamma \in (0,1)$ is the replanning triggering ratio.
Here, the performance metric $\eta^\star_t$ is computed via \(\texttt{Eval}(\cdot)\)
given the latest observations.
For completed tasks,
the expected value $\eta^\star$ is updated as
\begin{equation}
\label{eq:update_eta}
\eta^\star_{+} \triangleq \frac{\big{(}|\varphi(t)| - |\Omega_\iota|
  +|\Omega_{\texttt{c}}(t)|\big{)}}  {|\varphi(t)| - |\Omega_\iota|} \eta^\star_{-},
\end{equation}
where $\eta^\star_{+}$ and $\eta^\star_{-}$ denote the updated and previous values;
$|\varphi(t)|$ is the total number of tasks;
$|\Omega_\iota|$ is the number of completed tasks;
and $|\Omega_{\texttt{c}}(t)|$ is the number of tasks finished at time $t$.
When replanning is triggered, CP-MCTS generates a candidate plan $\widehat{\Pi}_t$ with metric $\widehat{\eta}_t$.
The system adopts this new plan only if it offers improved performance as $\widehat{\eta}_t < \eta^\star_t$.
Otherwise, the current plan $\Pi^\star$ is maintained and 
the metric $\eta^\star$ is monotonically decreased to reflect progress.
This procedure repeats until the system terminates.

\subsubsection{Complexity Analysis}
The computational complexity of Alg. \ref{alg:ta} is analyzed as follows.
Generating a R-poset has worst-case complexity \(\mathcal{O}(J^2)\), where \(J\) is the number of subtasks,
bounded by the number of edges in the NBA.
For task assignment, the worst-case search space
is \(\mathcal{O}(J!\cdot N^J)\),
as subtask orderings are combinatorial and
assignments grow exponentially with the number of robots~$N$.
In contrast, the \textit{rollout} process remains
\(\mathcal{O}(J\cdot N)\), as subtasks are assigned randomly or greedily.
The complexity of stepwise simulation is \(\mathcal{O}(\frac{z T\cdot(J+N)}{\Delta t})\),
where \(z\) is the number of samples, \(T\) is the makespan
and \(\Delta t\) is the time step.

\subsubsection{Generalization}
There are two notable extensions of the proposed scheme.
\textit{(I) Robot Failures}.
In the event of robot failures,
a modified {CP-MCTS} is employed to generate a new plan.
The root node is updated to exclude failed robots and include completed subtasks.
If a failure occurs during execution,
the affected subtask is rescheduled.
The standard CP-MCTS procedure is then applied to update the assignment.
\textit{(II) Dynamic Task Priority}.
When tasks are associated with different priorities~\(\alpha_\ell\),
the optimization objective shifts from the average makespan~$\overline{T}_{\varphi}$
to the weighted makespan~$\widetilde{T}_{\varphi} \triangleq \sum_{\ell=1}^{L_t} \alpha_\ell T_{\varphi_\ell}$,
which emphasizes high-priority tasks.


\section{Numerical Experiments} \label{sec:experiments}
This section presents numerical validations of the proposed method
over large-scale multi-robot systems.
The implementation is in \texttt{Python3} and
tested on a laptop with an AMD Ryzen 9 7845HX CPU.
Simulation and experiment videos are available in the supplementary material.

\subsection{Scenario Description}\label{subsec:description}
For the wildlife protection scenario in Fig. \ref{fig:online main},
four types of animals: antelope, rabbit, elephant, and tiger,
move within an \(800m\times 800m\) outdoor area containing obstacles 
such as reservoirs and hills.
A total of $2000$ synthetic trajectories are collected for the four types of animals
as the training and calibration datasets ($1000$ each),
where initial and target positions are specified and Gaussian noise is added to sampled waypoints.
Four robots of each type are deployed:
fast GRs~\(G_{\texttt{f}}\) to monitor, patrol and arrest;
slow GRs~\(G_{\texttt{s}}\) to monitor, patrol and rescue;
and UAVs~\(U\) to monitor and film.
GRs can navigate safely around obstacles as a team.
Tasks are specified by the following LTL formulas, which require to
monitor, film, and rescue animals, while also patrolling and arresting poachers in designated areas:
\begin{equation*}
\begin{split}
  &\varphi_\texttt{static}  =\varphi_{\texttt{p}-\texttt{s}_1} \wedge \varphi_{\texttt{p}-\texttt{s}_2} \wedge
    \varphi_{\texttt{mf}-\texttt{a}} \wedge \varphi_{\texttt{mf}-\texttt{r}} \wedge \varphi_{\texttt{mf}-\texttt{e}},\\
    &\varphi_{\texttt{r}}(50)  = \varphi_{\texttt{mf}-\texttt{t}}, \;
    \varphi_{\texttt{r}}(90)  = \Diamond\texttt{arrest}_{\texttt{s}_3} \wedge \Diamond\texttt{rescue}_\texttt{e},
\end{split}
\end{equation*}
where \(\varphi_\texttt{p}\), \(\varphi_\texttt{mf}\) follow the same structure as in (\ref{eq:task}).
Two scenes with different animal motion patterns
are considered (Fig.~\ref{fig:online main}):
\textbf{Scene-1}: targets exhibit piecewise uniform linear motion;
\textbf{Scene-2}: targets follow general smooth trajectories.
The main parameters are: CP failure probability~$\delta=0.15$,
CP-MCTS time budget~\(t_b=10s\), exploration factor~\(Q=1.5\),
random factor~\(\epsilon=0.3\), number of samples~\(z=50\),
risk level~\(\alpha=0.05\), and replanning triggering ratio~\(\gamma =0.2\).

\begin{figure}[t]
  \centering
  \includegraphics[width=0.95\linewidth]{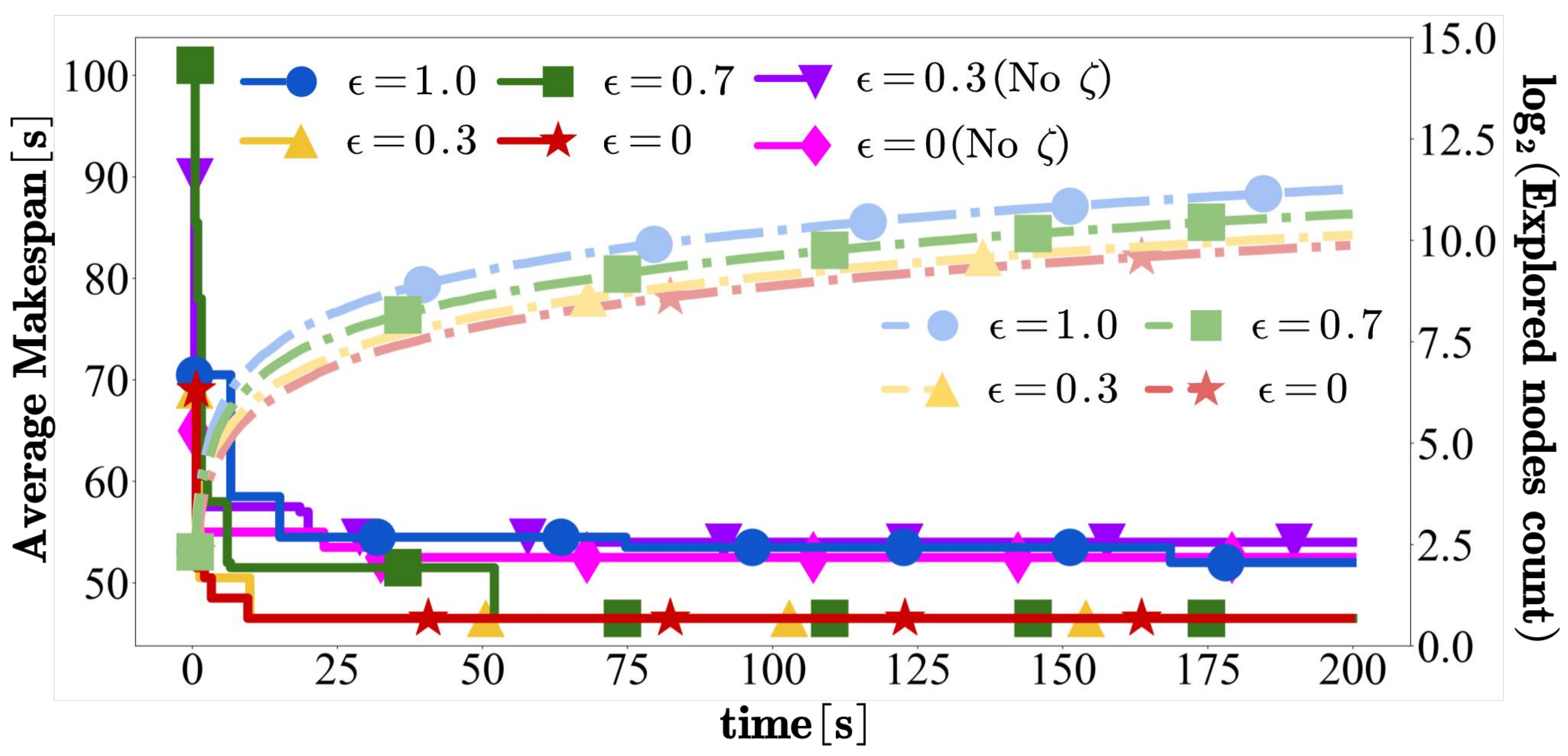}
  \vspace{-3mm}
  \caption{The average makespan and
    the number of explored nodes
    with different random factors~\(\epsilon\) and
    expansion strategies (with or without CP-based metric~$\zeta$ in (\ref{eq:metric}))
    during initial planning in Scene-1.
  }
  \label{fig:convergence}
  \vspace{-0.5mm}
\end{figure}

\subsection{Results}\label{subsec:results}
The results are shown in Figs. \ref{fig:online main} and \ref{fig:gantt}.
At \(t=0\),
trajectory prediction takes $0.05$s,
R-posets computation for $8$ static tasks takes $0.03$s,
and the ``Simulation'' at each node averages $0.25$s with $z=50$ samples.
The metric $\zeta_{\nu^+}$ in (\ref{eq:metric}) ranges from $24.5$s to $72.3$s,
and a plan with average makespan $\eta^\star_0=48$s is generated
after exploring $118$ nodes within $t_b=10$s.
At \(t=43s\), replanning is triggered due to performance degradation,
with $\eta^\star_{43} = 14.2$s exceeding the threshold $(1+\gamma)\eta^\star= 13.7$s for $\gamma=0.2$.
After exploring $140$ nodes, the new plan yields \(\widehat{\eta}_{43}=16.4\)s,
which does not improve performance; thus, the current plan is retained.
At \(t=50s\), the appearance of a new target \texttt{tiger} and 
task~$\varphi_{\texttt{mf-t}}$ triggers replanning
with $4$ subtasks, where $389$ nodes are explored,
producing \(\eta^\star_{50}=28.2\)s.
At \(t=90s\), reactive tasks $\texttt{arrest}_{\texttt{s}_3}$ and $\texttt{rescue}_{\texttt{e}}$ 
trigger replanning with $3$ unfinished subtasks,
exploring $155$ nodes and yielding \(\eta^\star_{90}=40.3\)s.
The system terminates at \(t=147\)s, with an average makespan of \(67.6\)s over $12$ subtasks.
Fig.~\ref{fig:convergence} illustrates convergence and node exploration at $t=0$.
The algorithm is deemed converged
when the best plan remains unchanged for more than $100$s.
As the random factor~\(\epsilon\) increases, more nodes are explored,
but convergence time grows.
With heuristic-based robot selection during \textit{rollout},
convergence is achieved within \(10\)s for \(\epsilon = 0\) and \(\epsilon = 0.3\).
Moreover, using $\zeta_{\nu^+}$ as in (\ref{eq:metric}) to filter child nodes
accelerates tree exploration and improves performance.

\begin{figure}[t]
  \centering
  \includegraphics[width=1.0\linewidth]{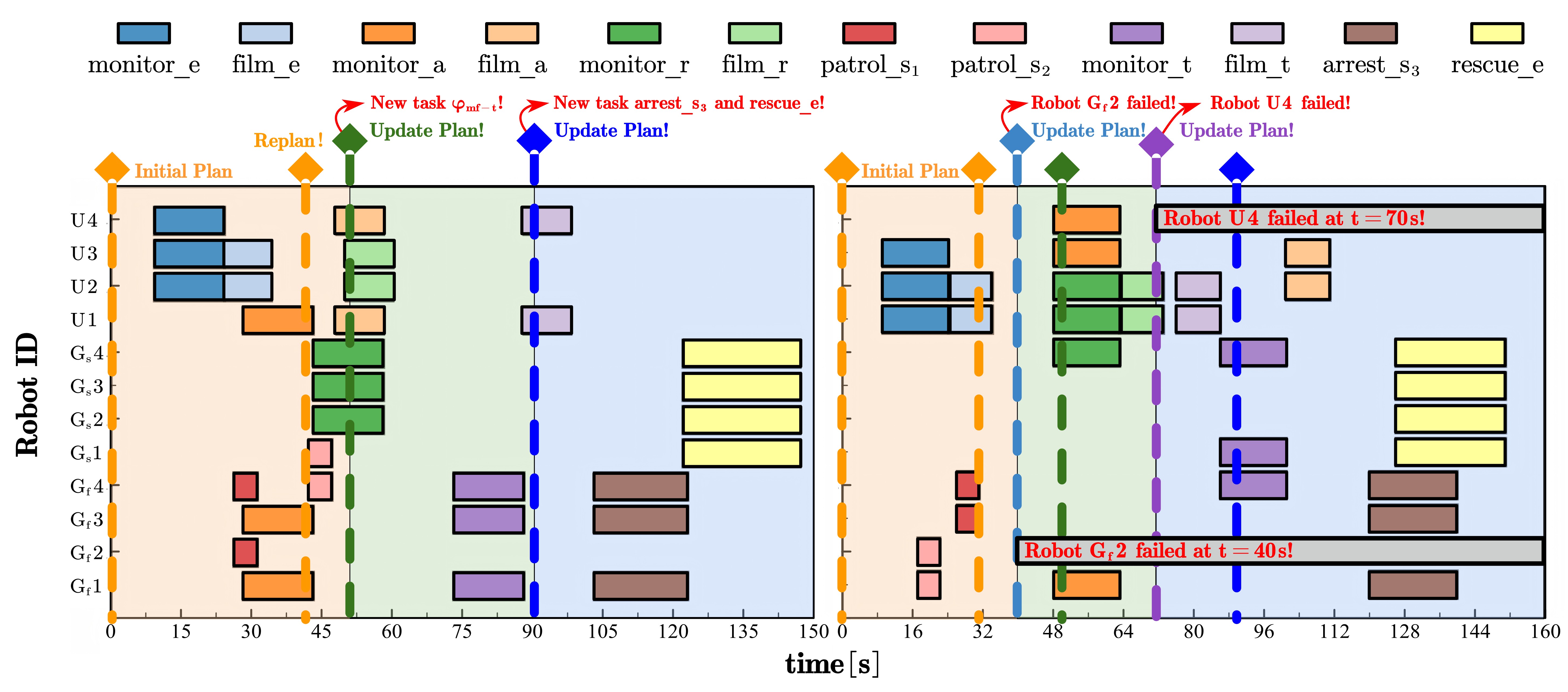}
  \vspace{-8mm}
  \caption{
    \textbf{Left:} Gantt chart of Scene-1,
    where replanning happens when the predicted value~$\eta^\star_t$ exceeds
    a threshold (orange line) and new tasks are triggered (green and blue lines).
    \textbf{Right:} Gantt chart of Scene-1
    where two robots fail at~$40s$ and $70s$ (in grey), respectively.}
  \label{fig:gantt}
\end{figure}

\subsection{Comparisons}\label{subsec:comparisons}
To evaluate the effectiveness of our method (\textbf{Ours}),
we compare against \textbf{six} baselines:
{Mixed Integer Linear Programming (\textbf{MILP})}:
task decomposition and assignment via integer optimization~\cite{luo2022temporal};
{Branch and Bound (\textbf{BnB})}:
a search-based method from~\cite{LIU2024111377};
{No Trajectory Prediction (\textbf{NTP})}
and {No Uncertainty (\textbf{NU})}:
{Ours} without trajectory prediction or without uncertainty regions, respectively;
\textbf{Ours-G}:
{Ours} with a GRU-based trajectory predictor;
\textbf{Ours-P10}:
Ours with periodic replanning every 10s instead of event-triggered replanning.
All methods are first evaluated for task allocation at \(t=0\) in Scene-1,
executed over $500$ target trajectory samples.
Metrics include the mean and variance of average makespan and the solution time.
They are further tested online in Scene-1 and Scene-2 with $500$ samples,
together with a clairvoyant strategy (\textbf{CS}) that has access to future trajectories
and applies NU for assignment.
Static methods ({MILP}, {BnB}, {NTP})
model each dynamic target as fixed at its most recently observed position during planning.

As shown in Table \ref{table:comp},
Ours outperforms static methods with reductions of
\(16.5\%\), \(24.6\%\), and \(22.5\%\) in mean makespan
and \(64.9\%\), \(67.3\%\), and \(74.8\%\) in variance
across Scene-1 ($t=0$), Scene-1, and Scene-2, respectively.
{BnB finds the first feasible solution quickly
but converges slowly due to random exploration.
{MILP} is the most time-consuming because of its exponential complexity.
{Ours} requires more computation than {NTP} and {BnB} due to sampling and 
stepwise simulation, but achieves more reliable outcomes.}


Although {NU} achieves a $6.7\%$ lower mean makespan than {Ours} in Scene-1 at $t=0$,
it incurs $1.2 \times$ higher variance and performs worse in online settings,
confirming that {Ours} offers more risk-averse assignments via CP.
{Ours-G} performs comparably to {Ours},
showing adaptability to different predictors.
{Ours-P10} increases mean makespan by $10.2\%$ and variance by $4.2\%$ relative to {Ours},
underscoring the benefit of event-triggered replanning.
Finally, Ours remains within $6.0\% \sim 7.5\%$ of CS in mean makespan with only minor variance increases,
indicating performance close to CS without prior knowledge of target motion.

\begin{table}[t]
  \begin{center}
  \begin{threeparttable}
    \caption{COMPARISON AGAINST BASELINES}\label{table:comp}
    \setlength{\tabcolsep}{0.20\tabcolsep}
    \centering
    \renewcommand{\arraystretch}{0.9}
    \begin{tabular}{c c c p{0.3mm} c c c c c}
       \toprule[1pt]
      \midrule
      \multirow{2}{*}{\textbf
{Method}} & \multicolumn{2}{c}{\textbf{Scene-1 ($\mathbf{t=0}$)}} & & \multicolumn{1}{c}{\textbf{Scene-1}} & & \multicolumn{1}{c}{\textbf{Scene-2}} \\
      \cmidrule{2-3} \cmidrule{5-5} \cmidrule{7-7}
      & \textbf{\makecell{M($\mathbf{\overline{T}_{\phi}}$), V($\mathbf{\overline{T}_{\phi}}$)}}\tnote{a}
      & $\mathbf{t_p}[s]\tnote{b}$
      & {} 
      & \textbf{\makecell{M($\mathbf{\overline{T}_{\phi}}$), V($\mathbf{\overline{T}_{\phi}}$)}}
      & {} 
      & \textbf{\makecell{M($\mathbf{\overline{T}_{\phi}}$), V($\mathbf{\overline{T}_{\phi}}$)}}
\\
      \midrule
      \textbf{Ours} & \underline{48.0}, \textbf{3.25} & 1.22, 4.70 & & 67.6, \underline{5.32} & & \underline{68.5}, \underline{4.34} \\[.1cm]
      {NU} & \textbf{44.8}, 7.05 & 0.30, \underline{2.74} & & 72.5, 14.32 & & 71.8, 12.59\\[.1cm]
      {NTP} & 56.2, 10.05 &  \underline{0.25}, \textbf{0.53} & & 81.9, 12.41 & & 83.0,  17.21\\[.1cm]
      {MILP} & 59.1, 8.79 & 3.24, 905.6 & & 94.9, 16.33  & & 94.8,  18.12 \\[.1cm]
      {BnB} & 57.1, 8.91  & \textbf{0.06}, 35.22 & & 92.0, 20.13  && 87.5, 16.24  \\[.1cm]
      {Ours-G} & 49.9, \underline{3.98} & 1.31, 5.21 & & \underline{66.9}, 6.12 & & 70.2, 5.43\\[.1cm]
      {Ours-P10} & / & / & & 79.3, 11.45 & & 77.8, 12.67\\[.1cm]
      {CS}  & /  & / & & \textbf{63.8}, \textbf{3.24} && \textbf{65.3}, \textbf{2.96}  \\[.1cm]
      \midrule
      \bottomrule[1pt]
    \end{tabular}
    \renewcommand{\arraystretch}{1.0}
    \begin{tablenotes}
  \footnotesize
    \item[a] \textbf{M($\mathbf{\overline{T}_{\phi}}$)}, \textbf{V($\mathbf{\overline{T}_{\phi}}$)}: 
    mean and variance of average makespan~$\overline{T}_{\phi}$.
    \item[b] The solution time~$\mathbf{t_p}$ is measured by two timestamps: 
      (i) when the first solution is returned;
      (ii) when the algorithm converges.
    \item[c] Best values are in \textbf{bold}; second-best are \underline{underlined}.
  \end{tablenotes}
  \end{threeparttable}
  \end{center}
  \vspace{-2mm}
\end{table}
\begin{table}[!t]
  \begin{center}
  \renewcommand{\arraystretch}{1.15}
  \begin{threeparttable}
    \caption{SCALABILITY ANALYSIS RESULTS}\label{table:scalability}
    \vspace{-0.05in}
    \setlength{\tabcolsep}{0.4\tabcolsep}
    \centering
    \begin{tabular}{c c c c c c c}
       \toprule[1pt]
      \midrule
      \multirow{2}{*}{\textbf{($\mathbf{G_{\texttt{f}}}$, $\mathbf{G_{\texttt{s}}}$, $\mathbf{U}$)}\tnote{a}}
      & \multicolumn{3}{c}{$\mathbf{t_p}$ ($\mathbf{M=3}$)$[s]$} & & \multicolumn{2}{c}{$\mathbf{t_p} (\mathbf{F=16}$)$[s]$} \\
      \cmidrule{2-4} \cmidrule{6-7}
      & $\mathbf{F=8}$
      & $\mathbf{F=12}$
      & $\mathbf{F=16}$
      & {}
      & $\mathbf{M=6}$
      & $\mathbf{M=9}$
      \\
      \midrule
      (4, 4, 4)   & 1.22, 4.7  & 2.03, 8.2  & 3.93, 15.5 & & 2.04, 13.2 & 1.63, 15.6 \\
      (8, 8, 8)   & 2.71, 4.1  & 5.15, 9.2  & 7.64, 18.1  & & 5.29, 12.1 & 7.59, 19.1 \\
      (12, 12, 12)   & 3.81, 4.9  & 6.19, 9.5 & 9.39, 18.3 & & 8.35, 12.8 & 9.59, 16.3 \\
      \midrule
      \bottomrule[1pt]
    \end{tabular}
    \begin{tablenotes}
    \footnotesize
      \item[a] \textbf{($\mathbf{G_{\texttt{f}}}$, $\mathbf{G_{\texttt{s}}}$, $\mathbf{U}$)}: the number of three types of robots.
    \end{tablenotes}
  \end{threeparttable}
  \end{center}
  \vspace{-3mm}
\end{table}

\subsection{Generalization}
\subsubsection{Scalability Analysis}
Scalability is evaluated with respect to
the number of robots $N$, tasks $F$, and targets $M$,
as summarized in Table~\ref{table:scalability}.
With fixed $F=8$ and $M=3$,
increasing $N$ from $12$ to $36$
raises the computation time for the first solution from~$1.22$s to $3.81$s,
while the convergence time remains around $4.7$s.
With fixed $M=3$ and $N=12$,
increasing $F$ from $8$ to $16$
extends convergence time from $4.7$s to $15.5$s,
whereas the time for the first solution grows only modestly from $1.22$s to $3.93$s,
owing to the polynomial complexity of \textit{rollout}.
With fixed $F=16$ and $N=12$,
increasing $M$ from $3$ to $9$
causes fluctuations in the first-solution time,
but convergence time stays around $15.0$s.

\subsubsection{Robot Failure}
In Scene-1, ground robot~\(G_{\texttt{f}}2\) fails at $t=40s$,
and UAV~\(U4\) at $t=70s$.
As shown in the Gantt chart of Fig. \ref{fig:gantt},
the failure of \(G_{\texttt{f}}2\) during antelope monitoring
triggers replanning, where \(U3\) replaces it.
At $t=70s$, the failure of UAV~\(U4\) during antelope filming
leads to another update,
assigning \(U2\) and \(U3\) to complete the task.
Despite these failures, the team successfully completes all tasks by~\(t=151s\), 
with an average makespan of \(74.9s\).

\subsubsection{Dynamic Task Priority}
In Scene-1,
the antelope-related task~\(\varphi_{\texttt{mf}-\texttt{a}}\) is assigned a priority coefficient of \(0.3\),
while other tasks are given \(0.1\).
When optimizing for the average makespan~$\overline{T}_{\varphi}$,
the completion times for monitoring and filming the antelope are $T_{\texttt{m}-\texttt{a}}= 92s$ and $T_{\texttt{f}-\texttt{a}}=112s$.
Under weighted makespan~$\widetilde{T}_{\varphi}$ minimization,
these times decrease to $82s$ and $102s$, respectively,
each reduced by $10s$.

\subsubsection{High-fidelity ROS Simulation}
In the second scenario (Figs.~\ref{fig:online main} and \ref{fig:ros_main}), three types of suspects:
on foot~$\texttt{f}$, by bike~$\texttt{b}$, and by car~$\texttt{c}$
flee through city roads.
Two types of robots, {four} of each, are deployed:
UAVs~$U_{\texttt{t}}$ for tracking suspects and scouting,
and UAVs~$U_{\texttt{i}}$ for interception and scouting.
Tasks are given by
  $\varphi_\texttt{static}  =\varphi_{\texttt{s}-\texttt{s}_1} \wedge \varphi_{\texttt{s}-\texttt{s}_2} \wedge
    \varphi_{\texttt{ti}-\texttt{f}} \wedge \varphi_{\texttt{ti}-\texttt{b}} \wedge \varphi_{\texttt{ti}-\texttt{c}}$,
    $\varphi_{\texttt{r}}(40)  = \varphi_{\texttt{s}-\texttt{s}_3}, \;
    \varphi_{\texttt{r}}(80)  =
    \varphi_{\texttt{s}-\texttt{s}_4},
    \varphi_{\texttt{s}-\texttt{s}_1} = \Diamond \texttt{scout}_{\texttt{s}_1}$,
    $\varphi_{\texttt{ti}-\texttt{f}} = \Diamond(\texttt{track}_\texttt{f}\wedge\neg\texttt{intercept}_\texttt{f}\wedge\Diamond\texttt{intercept}_\texttt{f})$.
Task assignments are updated at \(t=40s\) and \(t=80s\) when new reactive tasks appear.
Additional replanning occurs at \(t=64s\), \(t=79s\), and \(t=104s\), though the plan remains unchanged.
The system terminates at \(t=136s\),
completing all $10$ tasks with an average makespan of \(66.9s\).

\begin{figure}[t]
  \centering
  \includegraphics[width=1.0\linewidth]{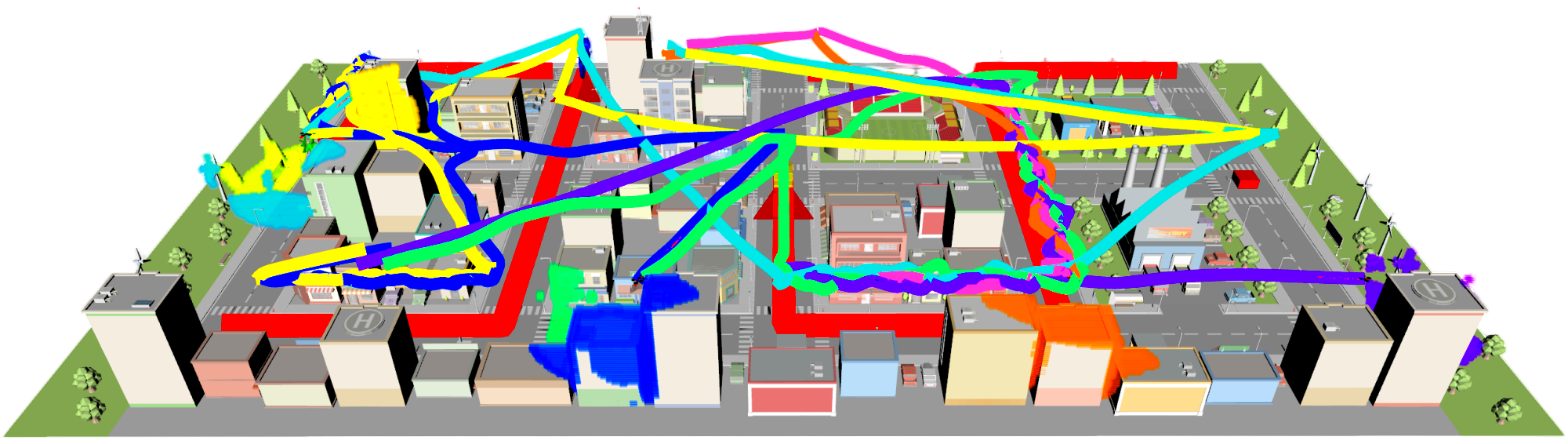}
  \hfil
  \includegraphics[width=0.95\linewidth]{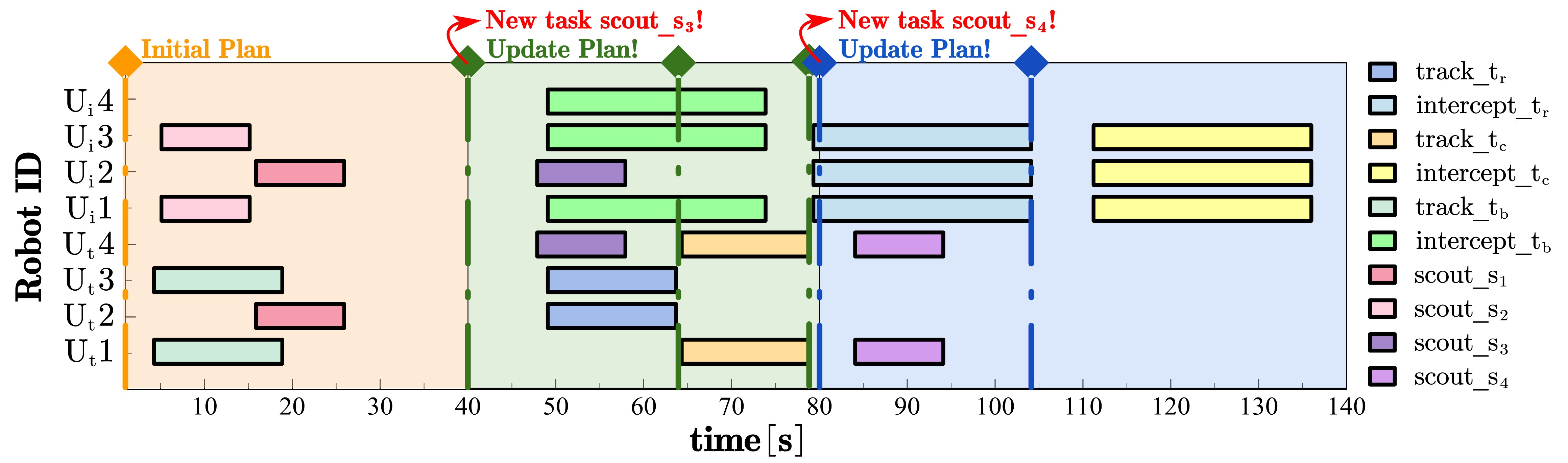}
  \vspace{-2.5mm}
  \caption{ROS simulation results.
  \textbf{Top}: Robot and target trajectories.
  \textbf{Bottom}: Gantt chart of the execution timeline with task allocation and replanning.
  }
  \label{fig:ros_main}
  \vspace{-0.5mm}
\end{figure}

\subsection{Hardware Experiments}
For further validation with hardware,
we built a setup similar to the second scenario with
$4$ robots (\texttt{UAV}: $2 U_{\texttt{t}}$, $2 U_{\texttt{i}}$)
and $2$ targets (\texttt{UGV}: $a_1$, $a_2$) in a $4.95m \times 4.95m$ workspace,
with the motion capture system \texttt{OptiTrack} providing their global states.
Each robot communicates wirelessly with a control PC via \texttt{ROS1}.
Mature navigation controllers are adopted and omitted for brevity.
The tasks are specified as:
$\varphi_\texttt{static}  =\varphi_{\texttt{s}-\texttt{s}_1} \wedge \varphi_{\texttt{s}-\texttt{s}_2} \wedge
  \varphi_{\texttt{ti}-\texttt{a}_1} \wedge \varphi_{\texttt{ti}-\texttt{a}_2}$,
$\varphi_{\texttt{r}}(80)  = \varphi_{\texttt{s}-\texttt{s}_3}.$
The algorithm uses the same parameters as in simulation.
Assignments are updated at \(t=47s\) (due to performance degradation) 
and \(t=80s\) (due to a new task).
The system terminates at \(t=113s\),
completing all $7$ tasks with an average makespan of \(64.7s\).
The resulting trajectory, snapshots, and Gantt chart are shown in 
Figs. \ref{fig:online main} and \ref{fig:exp_main}.

\section{Conclusion} \label{sec:conclusion}
This paper has presented \textbf{UMBRELLA}, an online multi-robot coordination framework
for collaborative temporal tasks with dynamic targets.
It achieves substantial reductions in both the mean and variance of the average makespan,
while guaranteeing satisfaction of spatial-temporal task specifications.
Future work includes intention-aware prediction for dynamic targets
and finer motion constraints for robots.

\begin{figure}[t]
    \centering
    \includegraphics[width=1.0\linewidth]{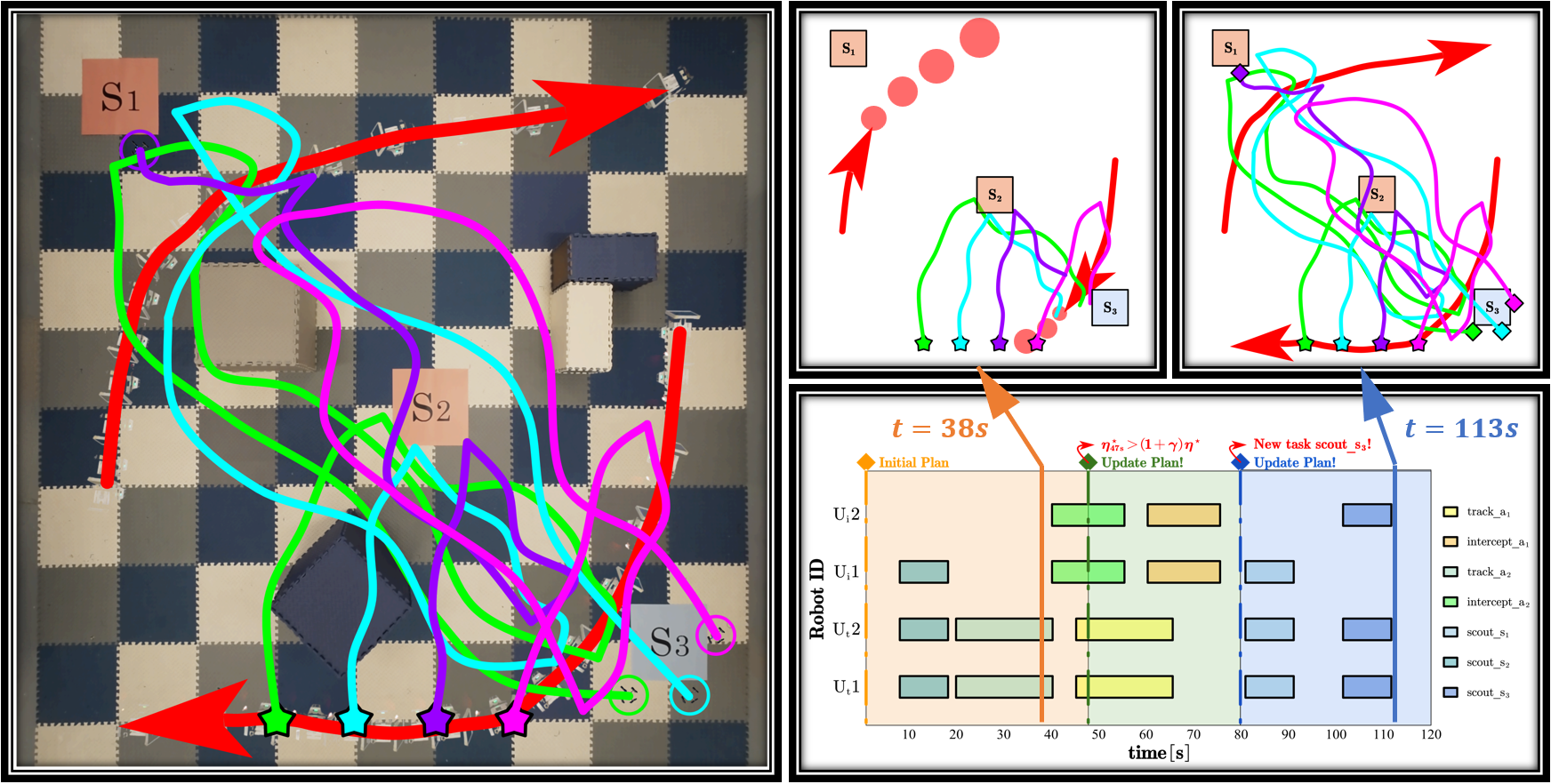}
    \vspace{-7.0mm}
    \caption{Hardware experiment results.
    \textbf{Left:} Recorded robot and target trajectories.
    \textbf{Top-right:} Example trajectory snapshots highlighting robot-target interactions.
    \textbf{Bottom-right:} Gantt chart of the execution timeline.
    }
    \label{fig:exp_main}
    \vspace{-0.5mm}
  \end{figure}



\bibliographystyle{IEEEtran}
\bibliography{contents/references}

@book{belta2017formal,
  title  = {Formal Methods for Discrete-Time Dynamical Systems},
  author = {Belta, Calin and Yordanov, Boyan and Gol, Ebru Aydin},
  volume = {15},
  year   = {2017},
  publisher={Springer}
}

@article{LIU2024111377,
  title   = {Time minimization and online synchronization for multi-agent systems under collaborative temporal logic tasks},
  journal = {Automatica},
  volume  = {159},
  pages   = {111377},
  year    = {2024},
  issn    = {0005-1098},
  author  = {Zesen Liu and Meng Guo and Zhongkui Li}
}

@inproceedings{2006Bandit,
  title     = {Bandit based monte-carlo planning},
  author    = {Kocsis, Levente and Szepesv{\'a}ri, Csaba},
  booktitle = {Eur. Conf. Mach. Learn.},
  pages     = {282--293},
  year      = {2006}
}

@article{luo2022temporal,
  title   = {Temporal Logic Task Allocation in Heterogeneous Multi-Robot Systems},
  author  = {Luo, Xusheng and Zavlanos, Michael M},
  journal = {IEEE Trans. Robot.},
  year    = {2022},
  volume  = {38},
  number  = {6},
  pages   = {3602-3621}
}

@article{schillinger2018simultaneous,
  title   = {Simultaneous task allocation and planning for temporal logic goals in heterogeneous multi-robot systems},
  author  = {Schillinger, Philipp and B{\"u}rger, Mathias and Dimarogonas, Dimos V},
  journal = {Int. J. Robot. Res.},
  volume  = {37},
  number  = {7},
  pages   = {818--838},
  year    = {2018}
}

@article{kantaros2020stylus,
  title   = {Stylus*: A temporal logic optimal control synthesis algorithm for large-scale multi-robot systems},
  author  = {Kantaros, Yiannis and Zavlanos, Michael M},
  journal = {Int. J. Robot. Res.},
  volume  = {39},
  number  = {7},
  pages   = {812--836},
  year    = {2020}
}

@book{baier2008principles,
  title     = {Principles of Model Checking},
  author    = {Baier, Christel and Katoen, Joost-Pieter},
  year      = {2008},
  publisher = {MIT Press}
}

@article{sahin2019multirobot,
  title   = {Multirobot coordination with counting temporal logics},
  author  = {Sahin, Yunus Emre and Nilsson, Petter and Ozay, Necmiye},
  journal = {IEEE Trans. Robot.},
  volume  = {36},
  number  = {4},
  pages   = {1189--1206},
  year    = {2019}
}

@article{toth2002overview,
  title     = {An overview of vehicle routing problems},
  author    = {Toth, Paolo and Vigo, Daniele},
  journal   = {The Vehicle Routing Problem},
  pages     = {1--26},
  year      = {2002},
  publisher = {SIAM}
}

@inproceedings{varava2017herding,
  title     = {Herding by Caging: a Topological Approach towards Guiding Moving Agents via Mobile Robots.},
  author    = {Varava, Anastasiia and Hang, Kaiyu and Kragic, Danica and Pokorny, Florian T},
  booktitle = {Robotics Sci. Syst.},
  pages     = {696--700},
  year      = {2017}
}

@inproceedings{cliff2015online,
  title     = {Online localization of radio-tagged wildlife with an autonomous aerial robot system},
  author    = {Cliff, Oliver M and Fitch, Robert and Sukkarieh, Salah and Saunders, Debra L and Heinsohn, Robert},
  booktitle = {Robotics Sci. Syst.},
  year      = {2015}
}

@article{guo2018multirobot,
  title     = {Multirobot data gathering under buffer constraints and intermittent communication},
  author    = {Guo, Meng and Zavlanos, Michael M},
  journal   = {IEEE Trans. Robot.},
  volume    = {34},
  number    = {4},
  pages     = {1082--1097},
  year      = {2018},
  publisher = {IEEE}
}

@article{guo2015multi,
  title     = {Multi-agent plan reconfiguration under local LTL specifications},
  author    = {Guo, Meng and Dimarogonas, Dimos V},
  journal   = {Int. J. Robot. Res.},
  volume    = {34},
  number    = {2},
  pages     = {218--235},
  year      = {2015},
  publisher = {SAGE Publications Sage UK: London, England}
}

@article{shukla2016application,
  title     = {Application of robotics in offshore oil and gas industry—A review Part II},
  author    = {Shukla, Amit and Karki, Hamad},
  journal   = {Robotics Auton. Syst.},
  volume    = {75},
  pages     = {508--524},
  year      = {2016},
  publisher = {Elsevier}
}

@article{bock2007construction,
  title   = {Construction robotics},
  author  = {Bock, Thomas},
  journal = {Auton. Robots},
  volume  = {22},
  number  = {3},
  pages   = {201--209},
  year    = {2007}
}

@inproceedings{Sama2023,
  author    = {Kalluraya, Samarth and Pappas, George J. and Kantaros, Yiannis},
  booktitle = {IEEE Int. Conf. Robot. Autom.},
  title     = {Multi-Robot Mission Planning in Dynamic Semantic Environments},
  year      = {2023},
  pages     = {1630-1637}
}

@inproceedings{Zhang2024,
  author    = {Zhang, Yuqing and Kalluraya, Samarth and Pappas, George J. and Kantaros, Yiannis},
  booktitle = {IEEE Conf. Decis. Control},
  title     = {Reactive Planning for Teams of Heterogeneous Robots with Dynamic Collaborative Temporal Logic Missions},
  year      = {2024},
  pages     = {1599-1606}
}

@article{Angelopoulos2021CP,
  title   = {A gentle introduction to conformal prediction and distribution-free uncertainty quantification},
  author  = {Angelopoulos, Anastasios N and Bates, Stephen},
  journal = {arXiv:2107.07511},
  year    = {2021}
}

@article{Lars2023CP,
  author  = {Lindemann, Lars and Cleaveland, Matthew and Shim, Gihyun and Pappas, George J.},
  journal = {IEEE Robot. Autom. Lett.},
  title   = {Safe Planning in Dynamic Environments Using Conformal Prediction},
  year    = {2023},
  volume  = {8},
  number  = {8},
  pages   = {5116-5123}
}

@article{Yu2023CP,
  title   = {Signal temporal logic control synthesis among uncontrollable dynamic agents with conformal prediction},
  author  = {Yu, Xinyi and Zhao, Yiqi and Yin, Xiang and Lindemann, Lars},
  journal = {arXiv:2312.04242},
  year    = {2023}
}

@inproceedings{tonkens2023,
  title={Scalable safe long-horizon planning in dynamic environments leveraging conformal prediction and temporal correlations},
  author={Tonkens, Sander and Sun, Sophia and Yu, Rose and Herbert, Sylvia},
  booktitle = {IEEE Int. Conf. Robot. Autom.},
  year={2023}
}

@article{Hochreiter1997LSTM,
  title   = {Long Short-term Memory},
  author  = {Hochreiter, S},
  journal = {Neural Comput.},
  year    = {1997}
}

@article{Liu2024fproduct,
  title     = {Fast and Adaptive Multi-Agent Planning under Collaborative Temporal Logic Tasks via Poset Products},
  author    = {Liu, Zesen and Guo, Meng and Bao, Weimin and Li, Zhongkui},
  journal   = {Research},
  volume    = {7},
  pages     = {0337},
  year      = {2024},
  publisher = {AAAS}
}

@article{robin2016multi,
  title   = {Multi-robot target detection and tracking: taxonomy and survey},
  author  = {Robin, Cyril and Lacroix, Simon},
  journal = {Auton. Robots},
  volume  = {40},
  pages   = {729--760},
  year    = {2016}
}

@article{rnn,
  title={Recurrent neural networks},
  author={Medsker, Larry R and Jain, Lakhmi and others},
  journal={Design and Applications},
  volume={5},
  number={64-67},
  pages={2},
  year={2001}
}

@article{gru,
  title={On the properties of neural machine translation: Encoder-decoder approaches},
  author={Cho, Kyunghyun and Van Merri{\"e}nboer, Bart and Bahdanau, Dzmitry and Bengio, Yoshua},
  journal={arXiv preprint arXiv:1409.1259},
  year={2014}
}

\vfill

\end{document}